\documentclass[runningheads]{llncs}

\authorrunning{H. N. Adorna}
\titlerunning{Matrices for SN P Systems}

\usepackage{amsmath}
\usepackage{amssymb}
\usepackage{mathtools}
\usepackage{graphicx}
\usepackage[colorinlistoftodos]{todonotes}
\usepackage{enumerate}
\usepackage{verbatim}
\usepackage{epsf}
\usepackage{latexsym}
\usepackage{cite}
\usepackage{setspace}
\usepackage{hyperref,multirow}
\usepackage{zed-csp}
\usepackage{pdflscape}

\usepackage{color}
\usepackage{todonotes}
\usepackage{amssymb,amsmath}
\usepackage{multirow}
\usepackage{pdflscape}
\usepackage{afterpage}

\usepackage{cite}
\usepackage{authblk}
\usepackage[T1]{fontenc}
\usepackage[utf8]{inputenc}

\newtheorem{mydef}{Definition}
\newtheorem{mythm}{Theorem}
\newtheorem{mylem}{Lemma}
\newtheorem{myobserve}{Observation}
\newtheorem{mycoro}{Corollary}

\newtheorem{myexample}{Example}

\newcommand{\myqed}{\begin{flushright}$\blacksquare$\end{flushright}}

	\title{Matrix Representations of Spiking Neural P Systems: Revisited}
\author{Henry N. Adorna}
	\institute{
	Department of Computer Science (Algorithm and Complexity)\\
	University of the Philippines Diliman\\
         1101 Quezon City, Philippines\\
	E-mail: {{\tt hnadorna@dcs.upd.edu.ph}}
	}

\begin{document}

\maketitle

\begin{abstract}
In the 2010, matrix representation of SN P system without delay was presented while in the case of SN P systems with delay, matrix representation was suggested in the 2017. These representations brought about series of simulation of SN P systems using computer software and hardware technology.  In this work, we revisit these representation and provide some observations on the behavior of the computations of SN P systems. The concept of reachability of configuration is considered in both SN P systems with and without delays.  A better computation of next configuration is proposed in the case of SN P system with delay. 
\end{abstract}


\section{Introduction}
Membrane computing or P system is introduced by Gh. P\u{a}un initially in the 1999 via electronic media and in the 2000, through the seminal journal paper \cite{paun2000}. P system could be classified as cell-like, tissue-like and neural-like. The books \cite{paun2002} and \cite{paun2010} are the best references for introducing one to P systems.  

In the 2006, neural-like P systems were introduced by Ionescu et al in \cite{ionescu2006}.  It was called Spiking Neural P system which was inspired by the neurophysiological behavior of neurons (in brain) sending electrical impulses along axons to other neurons.  Since the introduction of SN P system, several results  (see survey paper \cite{pan2016} and references herein, among others) have been reported in the literature, in particular, generating or recognizing elements of a set.  

In the 2010, motivated by finding the previous configuration in P system like in  \cite{naranjo2010}, the idea of representing SN P system by a matrix was introduced in \cite{zeng-etal2010}.  This somehow became the basis of several software simulations \cite{delacruz2018a, delacruz2018b, jimenez2018, delamor2017}, among others, done for SN P systems. 

\pagebreak

In the 2017, in response to a problem from \cite{zeng-etal2010}, a matrix representation of SN P system with delay is introduced in \cite{carandang2017}.  However, reference \cite{carandang2017} focused on the implementation of SN P system by a software in CUDA.

In this paper, we shall look into some details of these matrix representations of both SN P systems with and without delay.  

In Section \ref{1snp-d}, we introduced SN P sytems without delay and its matrix representation as in \cite{zeng-etal2010}.  We provided list of observations regarding the matrix representation. 
In Section \ref{snp-compute-mat}, the computation of SN P system without delay is presented, emphasizing results on reachability 

SN P systems with delay is revisited in Section \ref{snp+d}. The example mentioned in Sections thereafter is from \cite{carandang2017}.  Section \ref{snp+d-compute} presents a new formula for computing the next configuration different from \cite{carandang2017}.  Also reachability of configuration in SN P system with delay is provided 
We used the example SN P system in \cite {zeng-etal2010}, and \cite{carandang2017} to illustrate concepts and computations presented in this paper.
Section \ref{final} concludes the paper with remarks on the findings and some suggestions.


\section{SN P systems without Delay}\label{1snp-d}


\begin{mydef} {\bf (SN P Systems $\Pi$ without delay)} \\
An SN P system without delay, of degree $m \geq 1,$ is a construct of the form
$$\Pi = ( O, \sigma_1, \ldots , \sigma_m, syn, in, out),$$
where
\begin{enumerate}
\item $O= \{ a \}$ is a singleton alphabet ($a$ is called spike);
\item $\sigma_1, \ldots , \sigma_m$ are neurons of the form $$\sigma_i = (n_i, R_i), \, 1 \leq i \leq m,$$
where
\begin{enumerate}
\item $n_i \geq 0$ is the number of spikes in $\sigma_i;$
\item $R_i$ is a finite set of rules of the following forms:
\begin{enumerate}
\item[(1)] {\bf (Type 1;  spiking rule)} \\ $E / a^c \rightarrow a^p;d$ where $E$ is regular expression over $\{ a\},$ \\ and $c \geq 1, p \geq 1,$ such that $c \geq p,$ and a delay $d=0;$
\item[(2)]  {\bf(Type 2; forgetting rule)} \\  $a^s \rightarrow \lambda,$ for $s \geq 1,$ such that for each rule $E/a^c \rightarrow a^p ; 0$ of type (1) from $R_i,$  $a^s \not \in L(E);$
\end{enumerate}
\end{enumerate} 
\item $syn = \{\, (i,j) \mid 1 \leq i,j \leq m,\, i \not = j\, \}$ (synapses between neurons);
\item $in, out \in \{ 1,2, \ldots , m\}$ indicate the input and output neurons, respectively. 
\end{enumerate}
\end{mydef}

The SN P system $\Pi$ computes sequentially in the neuron level but simultatneaously in system level.  Computation starts (at time $k=0,$) with the {\bf initial configuration} (we give definition below) of the system represented by the amount of spikes present initially in each neuron.  To transition from one configuration to another one, the system applies simultaneously an applicable rule in each neuron at some (finite) time $k.$  

We say a rule in $R_i$ of neuron $i$ is {\bf applicable} at time $k,$ if and only if the amount of spikes in the neuron satisfies the $E_i$  or  $a_j^{(k)} \in L(E_{i}).$   At some time $k,$ we can have in neuron $\sigma_j,$ $a_j^{(k)} \in L(E_x) \cap L(E_y),$ for some rules $r_x,$ and $r_y,$  $x \not = y.$  Thus in the neuron level, among the set of all applicable rules at time $k,$ the system has to choose which of these rules will be applied at that time $k.$  This is how the nondeterminism of the system is realized.  Whereas several neurons with its chosen applicable rule can fire or spike simultaneouly at time $k,$ demonstrating parallelism. 

It is important to note that, rule of type 2 is the only rule that removes spike(s) from neuron at some time $k$ when applied.  And such type of rule could only be applied if and only if the amount of neurons is exactly the amount of spikes it needs to be applied.  
In particular, if there exist a rule $a^s \rightarrow  \lambda \in R_i,$ then there cannot exist any rule $E / a^c \rightarrow a^p; 0$ of type $1$ in $R_i$ such that $a^s \in L(E).$

In specifying rules in SN P system, we follow the standard convention of simply not specifying $E$ whenever the left-hand-side of the rule is equal to $E.$ 

The sequence of configuration will consists a {\bf computation} of the system. We say that computation halts or reaches {\bf halting configuration} if it reaches a configuration where no more applicable rules are available.  We could represent the output of the system either as the {\bf number of steps lapse between the first two spikes} of the designated output neuron or as {\bf spike train.}  The sequence of spikes made by the system until the system reaches a halting configuration is called  {\bf spike train.}

All SN P systems that we will consider in this Section are SN P system without delay unless stated otherwise.


\begin{myexample}\label{example1}  {\bf An SN P system without delay for $\mathbb{N} - \{1\}$} \cite{zeng-etal2010}\\
Let $\Pi = (\{a\}, \sigma_1, \sigma_2, \sigma_3, sys , out),$ where $\sigma_1 = ( 2, R_1),$ with $R_1 = \{ a^2 / a \rightarrow a,  a^2 \rightarrow a\};$  $\sigma_2 = (1, R_2),$ with $R_2 = \{ a \rightarrow a\};$ and $\sigma_3 =(1, R_3),$ with $R_3 = \{ a \rightarrow a, a^2 \rightarrow \lambda \};$ $syn = \{(1,2), (1,3), (2,1), (2,3)\};$  $out = \sigma_3.$

Note here that all rules do not have delay(s). This SN P system $\Pi$ produces or generates $\mathbb{N} - \{1\}.$ 
\end{myexample}

The object of this work is a matrix that represent SN P system, in particular in this section we consider SN P system without delay.  We state below the definition of matrix $M_{\Pi}$ for an SN P system without delay from \cite{zeng-etal2010}. 


\begin{mydef}\label{matrix-SNP-d} {\bf (Spiking Transition Matrix of $\Pi$ without delays)} \\
Let $\Pi$ be an SN P system without delay with $n$ total rules and $m$ neurons. The {\bf spiking transition matrix} of an SN P system $\Pi$ is $$M_{\Pi} = [ b_{ij} ]_{n \times m},$$
where
\[ b_{ij}  =
  \begin{cases}
    -c,       & \quad \text{if the left-hand side rule of } r_i  \text{ in } \sigma_j \text{ is } a^c\\
    & \quad \text{ and consumed } $c$ \text{ spikes}\\
   ~ ~ ~ p,  & \quad \text{if the right-hand side of the rule } r_i \text{ in } \sigma_s \\
   & \quad (s \neq j \text{ and } (s,j) \in syn) \text{ is } a^p\\ 
 ~ ~ ~ 0, & \quad \text{otherwise}  \end{cases}
\]

\end{mydef}

The matrix $M_{\Pi}$ is (almost) a natural representation of SN P system without delay.  Note that the entries $b_{ij}$'s are obtained by considering the rules $E_i / a^{c} \rightarrow a^{p}\colon 0$ in some neuron $\sigma_j.$  The rows of $M_{\Pi}$ are the set of rules, while the columns are the set of neurons.  At one glance matrix $M_{\Pi}$ provides information in the amount of spikes a certain rule consumed from a neuron where it resides and produced a spike to the adjacent neurons when a rule is applied.  Also, we could tell if a set of rules is in a neuron or not.  

In \cite{zeng-etal2010}, the matrix representation of Example \ref{example1} is 
$
M_{\Pi} = \left (
 \begin{matrix} 
  -1& 1 & 1 \\
  -2 & 1 & 1 \\
  1 & -1 & 1 \\ 
  0 & 0 & -1 \\
  0 & 0 & -2  
 \end{matrix}
\right ).
$  This matrix is a static representation of the system.


\subsection{Observations on $M_{\Pi}$}\label{snp-d-obs}

We provide summary of some observations on the matrix $M_{\Pi}$ that represents an SN P system $\Pi$ (without delay).

\begin{myobserve}\label{obs1}  \cite{zeng-etal2010} \\
Given matrix representation $M_{\Pi}$ of an SN P system $\Pi$ (without delay), we observed the following:
\begin{enumerate}
\item For each $i,$ there exists unique negative $b_{ij},$ for all $j.$ 
\item For each $j,$ there exists at least one negative $b_{ij},$ for all $i.$  
\end{enumerate}
\end{myobserve}

We add the following observations on network/graph structure to SN P system $\Pi$ through its matrix $M_{\Pi}.$

\begin{myobserve}\label{obs2}
Given matrix representation $M_{\Pi}$ of an SN P system $\Pi$ (without delay), we observed the following:
\begin{enumerate}
\item The number of negative entries for each $j,$ is the number of rules $\sigma_j$ has in $\Pi.$
\item $\sigma_j$ is designated as {\bf  output neuron,} 
if there exists a rule $r_i$ such that $b_{ij}$ is negative and is the only non-zero entry in row $i.$
\item If $b_{ij}$ is a negative entry in row $i,$  then $\sigma_j$ could provide/produce spike to $\sigma_{j'},$ 
for any existing non-zero positive entry, $b_{ij'},$ such that $j \neq j'.$  The number of such positive entries is the so-called {\bf out-degree of $\sigma_j$}
\end{enumerate}
\end{myobserve}

\begin{myobserve}\label{obs3}
Given matrix representation $M_{\Pi}$ 
of an SN P system $\Pi$ (without delay), we observed the following:
\begin{enumerate}
\item $M_{\Pi}$ is not always a square matrix.
\item $M_{\Pi}$ would be a square matrix, if and only if $m=n,$ that is, the number of neurons is equal to the total number of rules in the system. 
\end{enumerate}
\end{myobserve}

Number $1$ of Observation \ref{obs3} tells that the number of neurons could be less than the total number of rules in the system.  Also, it is possible to have more neurons than total number of rules in a system.

Number 2 of Observation \ref{obs3} would happen if each neuron would have a unique rule in it, or there are neurons which do not have rules in it but the system satisfies $m=n.$

\begin{myobserve}\label{obs4}
The matrix $M_{\Pi}$ 
of SN P system $\Pi$ (without delay) provides us the following information on the model: 
\begin{enumerate}
\item the amount of spikes each neuron $\sigma_j$ in the system consumes, if a rule in $\sigma_j$ is applied, i.e. $b_{ij} <0.$
\item the amount of spike each neuron $\sigma_j$ in the system produces, if a rule in $\sigma_j$ is applied, i.e. $b_{ij} \geq 0.$
\item the connectivity of neurons.
\end{enumerate}
\end{myobserve}

The matrix $M_{\Pi}$ of our example (from \cite{zeng-etal2010}) above could provide us the following static information about SN P system $\Pi$ without delay;  The existence of negative entries (namely, $-1$ and $-2$) in column $1$ tells us that neuron $\sigma_1$ has two rules; similarly, we could say, $\sigma_2$ has only one rule and $\sigma_3$ has two rule.  In particular, we can declare $\sigma_3$ an output neuron that sends spike(s) to the environment.  The positive entries in each row, provide us information about the amount of spike that a rule (a row in $M_{\Pi}$) sends to some neurons.  Note that with these information we could dentify the precise rule(s) in every neurons, but those in the output neuron(s).  This is when the so-called {\bf augmented spiking matrix for $\Pi,$} $M_{\Pi}$ \cite{zeng-etal2010} will be needed.  Augmented spiking matrix for $\Pi$ is the same matrix $M_{\Pi}$ but with an additional column to indicate environment.  Thus in the case of our example, we denote the  augmented spiking matrix for $\Pi$ as
$
\widehat{M_{\Pi}} = [b_{ij} | e_j]_{5 \times 4} = \left (
 \begin{matrix} 
  -1& 1 & 1  & 0\\
  -2 & 1 & 1 & 0\\
  1 & -1 & 1  & 0\\ 
  0 & 0 & -1  & 1\\
  0 & 0 & -2  & 0
 \end{matrix}
\right ).
$
 where the last column is one for the environment.  The $0$ in $(5,4)$-entry in $M_{\Pi}$ indicates that rule $r_5$ simply forgets, that is, $r_5: a^2 \rightarrow \lambda$ ($\lambda = a^0$).  While $(4,4)$-entry implies that $r_4: a \rightarrow a$ sends (output) spike to the environment.
 
 In this work we only consider $M_{\Pi}$ of the SN P system without delay. 

We consider in the following observation the structural topology of SN P system.
Observation \ref{obs5} below seems to reflect results obtained in \cite{Ibo2011} on existence of periodicty of SN P systems. 

\begin{myobserve}\label{obs5} 
Given an SN P system $\Pi$ (without delay), we observed the following:
\begin{enumerate}
\item Suppose we define matrix $$\text{\bf Struc-}M_{\Pi} = [ c_{ij} ]_{ m \times m},$$
such that for any pair of neuron $\sigma_i$ and $\sigma_j,$ we have 
\[ c_{ij}  =
  \begin{cases}
    -1,       & \quad (i,j) \in syn \text{ and } i = j\\   
   ~ ~ ~ 1,  & \quad (i,j) \in syn \text{ and } i \neq j\\
 ~ ~ ~ 0, & \quad (i,j) \not \in syn.
  \end{cases}
\]

Then $\text{\bf Struc-}M_{\Pi}$ represents a directed graph for $\Pi.$  In particular, $\text{\bf Struc-}M_{\Pi}$ is a square matrix.
\item If $\text{\rm row-rank}(\text{\bf Struc-}M_{\Pi})\footnote{$\text{\rm row-rank}(\text{\bf Struc-}M_{\Pi})$ is the number of linearly independent rows in $\text{\bf Struc-}M_{\Pi}$} < m,$ where $m$ is number of neurons, then $\Pi$ has a cycle.
\end{enumerate}
\end{myobserve}

In the case of our example,  $\text{\bf Struc-}M_{\Pi} =  \left (
 \begin{matrix} 
  -1& 1 & 1 \\
  1 & -1 & 1 \\ 
  0 & 0 & -1 
 \end{matrix}
\right ).
$

Notice that,  $\text{\rm row-rank}(\text{\bf Struc-}M_{\Pi}) = 2 < 3,$ and there is a loop created by the synapses connecting $\sigma_1$ and $\sigma_2.$

\section{Computation of SN P Systems via $M_{\Pi}$}\label{snp-compute-mat}

The configuration of an SN P system $\Pi$ is described by the amount of spikes present in each neuron in the system at a given instance. We define a 
configuration of $\Pi$ at time $k,$ as an $m$-vector of integers representing amount of spikes present in each $m$ neurons at a time $k.$   


\begin{mydef}{\bf (Configuration Vector of SN P systems $\Pi$ without delay)}\\
A {\bf configuration} of SN P system $\Pi$ without delay is an $m$-size vector of integers denoted by $$C = (a_1, a_2, \ldots , a_m),$$ where $a_j$'s are the amount of spikes present at neuron $j,$ for each $j=1, \ldots , m.$ 
  
We denote a {\bf configuration at time $k \geq 0$} of an SN P system $\Pi,$ with $m$ neurons as vector $$C^{(k)} = (a_1^{(k)}, a_2^{(k)}, \ldots , a_m^{(k)}),$$ where $a_j^{(k)} \in \mathbb{Z}^{+} \cup \{0\}$ for all $j=1, \ldots , m$ are the amount of spikes present at neuron $i$ at time $k.$ 

The vector $C^{(0)}=(a_1^{(0)}, a_2^{(0)} , \ldots , a_m^{(0)})$ is the initial configuration vector of SN P system $\Pi,$ where $a_1^{(0)}, a_2^{(0)} , \ldots , a_m^{(0)}$ are initial amount of spikes in neuron $\sigma_j,$ $j = 1, 2, \ldots , m.$ 
\end{mydef}

In some variants of SN P system, especially where SN P systems are used to solve practical or industrial problems, the number of spikes is (almost always) non-negative real numbers.  
In this work, we limit the number of spikes to be in $\mathbb{Z}^{+} \cup \{0\}.$

In Example \ref{example1}, the initial configuration of $\Pi$ is $C^{0)} = (2,1,1).$


\begin{mydef}\label{SpikingVector}{\bf (Spiking vector)}\\
Let $C^{(k)} = (a_1^{(k)}, a_2^{(k)}, \ldots , a_m^{(k)}),$ be the current configuration of SN P system $\Pi$ with $m$ neurons and a total of $n$ rules.  Assume a total order $d:$ $1,2, \ldots , n$ is given for all $n$ rules of $\Pi, $ so the rules can be refered to as $r_1, r_2, \ldots, r_n.$  We denote a {\bf valid spiking vector} by
$$Sp^{(k)} = (sp_1^{(k)}, sp_2^{(k)}, \ldots  , sp_n^{(k)}),$$ 
where
\begin{enumerate}  
\item for each neuron $\sigma_j,$ such that $a^{(k)}_j \in L(E_i),$ there exists a unique $r_i \in R_j$ of $\sigma_j,$ such that $sp^{(k)}_i = 1;$  \\ ($r_i$ is chosen among all possible applicable rules from $R_j$ at time $k.$)
\item if  $a^{(k)}_j \not \in L(E_i),$ then $sp^{(k)}_i = 0;$  and 
\item $\sum_{i=1}^{n} sp^{(k)}_i \leq m.$ 
\end{enumerate}

We denote by $Sp^{(0)} = (sp_1^{(0)}, sp_2^{(0)}, \ldots  , sp_n^{(0)}),$ the {\bf initial spiking vector} with respect to $C^{(0)}$ of $\Pi.$
\end{mydef}

Definition \ref{SpikingVector} implies that $Sp^{(k)} = (sp_1^{(k)}, sp_2^{(k)}, \ldots  , sp_n^{(k)}) 
$ is a vector such that 
$sp^{(k)}_i \in 
\{0,1\}$ for each $i=1,2, \ldots, n.$   The vector $Sp^{(k)}$ indicates which rules must be fired or used at time $k$ given current configuration $C^{(k)}$ of $\Pi.$  Moreover, if a neuron contains at least two rules and both rules could be applied at some time $k,$ then only one of the two rules must be selected to fire or spike at time $k.$   This is where non-determinism is demonstrated in the system.  A vector is not a valid spiking vector if it does not satisfy Definition \ref{SpikingVector}.

In our example, the initial spiking vector with respect to $C^{(0)}$ could be either $Sp^{(0)} = (1,0,1,1,0)$ or $(0,1,1,1,0).$   The vector $(1,1,1,1,0)$ is not a valid spiking vector.

Applying the 
rules as indicated by $Sp^{(k)}$ at time $k,$ would change the configuration from $C^{(k)}$ to $C^{(k+1)}.$  This means that neurons of  $\Pi$ in the next configuration would either gain or loss some spikes at time $k+1.$  In particular, we could represent such so-called {\bf transition net gain or loss} of $\Pi$ at time $k$ as vector $NG^{(k)}=C^{(k+1)} - C^{(k)}.$

 The product  $Sp^{(k)} \cdot M_{\Pi}$ represents the net amount of spikes the system obtained at time $k.$ The sum $sp^{(k)} _1 b_{1j} + \cdots + sp^{(k)}_n b_{nj}$ is the amount of spikes in $\sigma_j$ (for all $j$) at time $k.$
 


\begin{mylem} {\bf (Net gain vector)} \cite{zeng-etal2010}\\
Let $\Pi$ be an SN P systems without delay with $n$ total rules and $m$ neurons. Then
$$NG^{(k)} = Sp^{(k)} \cdot M_{\Pi},$$
where $Sp^{(k)}$ is a valid spiking vector.
\end{mylem}


\begin{mythm}\label{NextConfig} {\bf (Next configuration)} \cite{zeng-etal2010}\\
Let $\Pi$ be an SN P systems with total of  $n$ rules  and $m$ neurons. Let $C^{(k-1)}$ be given, then 
$$C^{(k)} = C^{(k-1)} + Sp^{(k-1)} \cdot M_{\Pi},$$
where  $Sp^{(k-1)}$ is the valid spiking vector with respect to $C^{(k-1)},$ and $M_{\Pi}$ is the spiking transition  matrix of $\Pi.$ 
\end{mythm}

\begin{mydef}\label{ValidConfig} {\bf (Valid Configuration $C$)}\\
A configuration $C$ of some SN P system $\Pi$ is {\bf valid} if and only if $C=C^{(0)},$ or there exist valid configuration $C'$ and valid spiking vector $Sp',$ such that $C= C' + Sp' \cdot M_{\Pi},$  where $M_{\Pi}$ is the matrix of SN P system $\Pi.$
\end{mydef}

Computation of SN P system $\Pi$ of Example \ref{example1} has been demonstrated in \cite{zeng-etal2010} to be faithfull as to how SN P system without delay works.


Using the formula from Theorem \ref{NextConfig}, we can obtain some valid configuration $C^{(k)}$ from the initial configuration.


\begin{mycoro}\label{Ck-by-C0}
Let $\Pi$ be an SN P systems with total of  $n$ rules  and $m$ neurons.  Then at any time $k,$ we have
$$C^{(k)} = C^{(0)} +\left ( \displaystyle \sum_{i=0}^{k-1}Sp^{(i)} \right )\cdot M_{\Pi},$$ where $Sp^{(i)}$ are valid spiking vectors for all$i=1,2, \ldots, k-1.$
\end{mycoro}
\begin{proof}
By Theorem \ref{NextConfig} we can say that for all $i \geq 0,$ $C^{(i+1)} = C^{(i)} + Sp^{(i)} \cdot M_{\Pi},$ for a valid  $Sp^{(i)},$ for each $i.$   Then we have

\begin{tabular}{ccl}
$C^{(k)}$ & = & $C^{(k-1)} + Sp^{(k-1)} \cdot M_{\Pi}$ \\
& = & $C^{(k-2)} + Sp^{(k-2)} \cdot M_{\Pi} + Sp^{(k-1)} \cdot M_{\Pi}$ \\
& = & $C^{(0)} + Sp^{(0)} \cdot M_{\Pi} + \cdots  + Sp^{(k-2)} \cdot M_{\Pi} + Sp^{(k-1)} \cdot M_{\Pi}$ \\
& =& $\vdots$ \\
& = & $C^{(0)} + (Sp^{(0)}  + Sp_{(1)} + \cdots  + Sp^{(k-2)}  + Sp^{(k-1)}) \cdot M_{\Pi}$ \\
$C^{(k)}$ & = & $C^{(0)} +\left ( \displaystyle \sum_{i=0}^{k-1}Sp^{(i)} \right )\cdot M_{\Pi},$
\end{tabular}

as required.
\myqed
\end{proof}


We use Corollary \ref{Ck-by-C0} to define reachability of a configuration $C$ of some SN P system $\Pi.$

\begin{mydef} {\bf (Reachability)}\\
Let $M_{\Pi}$ be the spiking transition  matrix of $\Pi,$ and $C^{(0)}$ its initial configuration,
A configuration $C$ is said to be {\bf $k$ reachable} in $\Pi$ if and only if there is in $\Pi$ a sequence of $k$ valid configurations $\{C^{(i)}\}_{i=0}^{k},$ for some $k,$ 
 such that $C^{(i)} = C^{(i-1)} + Sp^{(i-1)} \cdot M_{\Pi},$ 
 and $C^{(k)} =C.$   Moreover, $C$ must be valid configuration.
\end{mydef}

\begin{mythm}\label{reachability} {\bf (Reachability of $C$)} \\
Let $M_{\Pi}$ be an $n \times m$ matrix for SN P system $\Pi,$ and $C^{(0)}$ is its initial configuration,

Then a configuration $C$ is $k$ reachable in $\Pi$ if and only if $C = C^{(0)} + \overline{s} \cdot M_{\Pi},$ for some $\overline{s}$ where 
$\overline{s}$ is the sum of $k$ valid spiking vectors and $C^{(k)} =C.$
\end{mythm}

\begin{proof} $( \Longrightarrow)$ \quad 
Let $C$ be $k$ reachable in $\Pi.$   Then there is a sequence of valid configurations $\{ C^{(i)} \}_{i=0}^{k-1},$ such that for all $i \geq 0,$ $C^{(i+1)} = C^{(i)} + Sp^{(i)} \cdot M_{\Pi},$ for some valid $Sp^{(i)}.$  By definition and Corollary \ref{Ck-by-C0} , $$C=C^{(k)} = C^{(0)} +\left ( \displaystyle \sum_{i=0}^{k-1}sp^{(i)} \right )\cdot M_{\Pi}.$$
Let $\overline{s} = \sum_{i=0}^{k-1}Sp^{(i)}.$

$(\Longleftarrow)$ \quad 
Suppose there is a vector $\overline{s},$ such that  $C = C^{(0)} + \overline{s} \cdot M_{\Pi},$ where $C^{(0)}$ the initial configuration.  We need to find $k$ sequence of valid configurations $\{C^{(i)}\}_{i=0}^{k},$ such that $C^{(i)} = C^{(i-1)} + Sp^{(i-1)} \cdot M_{\Pi},$ 
 and $C=C^{(k)}.$   
 
A vector $\overline{s}$ can be obtained by solving the system of linear equations  $C - C^{(0)} = \overline{s} \cdot M_{\Pi}.$ 
We can write $\overline{s}$ as a sum:
$\overline{s} = s_{0} + s_{1} + s_{2} + \cdots + s_{k-1},$ where $s_{i} \in \{0,1\}^{n},$ for all $i=0,1, \ldots, k-1.$  There are finitely many possibilities to express $\overline{s}$ as sum of $\{0,1\}$-vectors.  Further,  $s_{i} \in \{0,1\}^{n}$ are the valid spiking vectors producing valid configurations.

We now construct a $k$ sequence of valid configurations from $C^{(0)}.$ The initial spiking vector $Sp^{(0)}$ can be easily obtained from $\Pi.$  We proceed as follows:

Given $\overline{s},$ $\Pi,$ and $C^{(0)},$  for $i=0,1, \ldots , k-1,$
\begin{enumerate}
\item  Compute for $Sp^{(i)}$ with respect to (w.r.t.) $C^{(i)}.$
\item  Obtain $C^{(i+1)}= C^{(i)} + Sp^{(i)} \cdot M_{\Pi}.$
\item  Calculate $\overline{s}^{i+1} = \overline{s}^{i} - Sp^{(i)}.$
\item  If $\overline{s}^{i+1}$  is already a $\{0,1\}$-vector, then check if it is a valid configuration vector w.r.t. $C^{(i+1)}.$
\begin{enumerate}
\item If $\overline{s}^{i+1}$ is a valid configuration w.r.t. $C^{(i+1)},$ Then
\begin{enumerate}
\item If $C=C^{(k)},$ where $k=i+1,$ then STOP and decide $C$ is reachable.
\item Otherwise, STOP and decide $C$ is not reachable in $\Pi.$ 
\end{enumerate}
\end{enumerate}
\item Otherwise, check if there is at least an entry in $\overline{s}^{(i+1)}$ that is a negative integer.
\begin{enumerate}
\item If there is a negative integer, then $\overline{s}^{i+1}$ is cannot be sum of valid spiking vectors.
\item Then STOP and decide, $C$ is not reachable in $\Pi.$
\item Otherwise, we repeat the process from Step $1$ for $C^{(i+1)}.$ 
\end{enumerate}
\end{enumerate}
If the difference obtained from Step $3$  becomes $0,$ after some integer $k,$ then $C$ is reachable using valid set of spiking vectors.  In particular, let  $\{C^{(i)}\}_{i=0}^{k},$ be such $k$-sequence of valid configurations computed during the process of computing for $C^{(k)}=C.$.
\myqed
\end{proof}

We provide below some examples of reachability computations using SN P system used in Example \ref{example1} (from  \cite{zeng-etal2010}) for illustrations.

\begin{myexample} \label{viaNextConfigThm}
Let $C=(1,1,2)$ be some configuration in $\Pi$ of Example \ref{example1}.  Suppose $C=(1,1,2)$ is reachable.  Then using Theorem \ref{NextConfig} we could obtain sequence of valid configuration from $C^{(0)}$ leading to $C^{(k)}=C,$ for some $k.$  In particular, at $C^{(0)}=(2,1,1),$ we have $C^{(1)}=(2,1,2)=(2,1,1) + (1,0,1,1,0) \cdot M_{\Pi},$ and   $C^{(2)}=C=(1,1,2)= (2,1,2) + (0,1,1,0,1) \cdot M_{\Pi}.$  Also, $C=C^{(2)} = (2,1,1) +( (1,0,1,1,0) + (0,1,1,0,1)) \cdot M_{\Pi}.$  In this case,  $C=(1,1,2)$ is $2$-reachable. 
\end{myexample}

It is not hard to see that $C=(1,1,2)$ is actually $1$-reachable via $Sp^{(0)}= (0,1,1,1,0).$ 

\begin{myexample}

Let $C=(2,1,2)$ be some (possible) configuration of SN P system $\Pi,$  such that  $C= C^{(0)} + \overline{s} \cdot M_{\Pi},$ for some $\overline{s},$ where $C^{(0)} = (2,1,1).$  To find $\overline{s}$ we solve the following system of linear equations obtained from $(2,1,2) - (2,1,1) = \overline{s} = (x,y,z,w,v) \cdot M_{\Pi}.$

\begin{tabular}{lcl}
$-x-2y+z$ & $=$ & $0$ \\
~ $x+y-z$ & $=$ & $0$ \\
~ $x+y+z -w-2v$ & $=$ & $1$ 
\end{tabular}

The solutions to this system of linear equations would be $(x,y,z,w,v)=(r,0,r,2(r-s)-1,s),$ for some integers $r,$ and $s,$ where $r > s.$   We consider only few instances of values of $r,$ in particular, when $r=1,$ and $r=2.$  To obtain a sequence of valid configurations, we follow the procedure in Theorem \ref{reachability}. 

If $r=1,$ then $s=0.$ Thus, $\overline{s}= (1,0,1,1,0)$ which is precisely a initial spiking vector for $\Pi.$   Therefore, $C=(2,1,2)$ is $1$-reachable in $\Pi,$  since $(2,1,2) = (2,1,1) + (1,0,1,1,0) \cdot M_{\Pi}.$  In particular, we have the following sequence of valid configurations: $\{(2,1,1), (2,1,2)\}.$

If $r=2,$ then $s=0,1.$  At $r=2, s=0,$  we have $\overline{s}= (2,0,2,3,0).$  We use the procedure from Theorem \ref{reachability} and proceed by using the folloiwng table:  Let $\overline{s}^{(0)}=\overline{s}.$

$$
\begin{tabular}{|c|c|c|c|}
\hline
$i$ \quad & \quad $\overline{s}^{(i)} - Sp^{(i)}$ \quad  & \quad $Sp^{(i)}$ \quad & $C^{(i)}$ \\
\hline
 $0$ & $(2,0,2,3,0)$ & $(1,0,1,1,0)$ & $(2,1,1)$\\
 
 $1$ & $(1,0,1,2,0)$ &  $(1,0,1,0,1)$ & $(2,1,2)$ \\
 
 $2$ & $(0,0,0,2,-1)$ &  &  \\
 & (not a valid sum vector)& & \\
 \hline
\end{tabular}
$$
 Notice that the bottom entry in column $\overline{s}^{(i)} - Sp^{(i)}$ has a vector with negative entry.  This indicates that the computation leads to an invalid vector sum, thus no valid configuration would be produced. 
 
 If we use the other possible initial vector $(0,1,1,1,0)$ and do the same process, we will also have similar observation.  Therefore, at instance $r=2,$ and $s=0,$ $(2,1,2)$ is not reachable in $\Pi.$ 
 
 Now at $r=2, s=1,$ with $\overline{s} = (2,0,2,1,1).$
 
 $$
\begin{tabular}{|c|c|c|c|}
\hline
$i$ \quad & \quad $\overline{s}^{(i)} - Sp^{(i)}$ \quad  & \quad $Sp^{(i)}$ \quad & $C^{(i)}$ \\
\hline
 $0$ & $(2,0,2,1,1)$ & $(1,0,1,1,0)$ & $(2,1,1)$\\
 
 $1$ & $(1,0,1,0,1)$ & $(1,0,1,0,1)$ & $(2,1,2)$ \\
 
 $2$ & $\text{\bf 0}$ &  &  $(2,1,2)$ \\
 \hline
\end{tabular}
$$

This computation implies that $C=(2,1,2)$ is reachable in $\Pi.$  If we use the other initial vector of $\Pi,$ then it will lead us to a vector sum with negative entry.

The table above, shows the sequence of valid configurations from $C^{(0)}$ leading to $C=(2,1,2)$ at $k=2$ is $\{ (2,1,1), (2,1,2),(2,1,2)\}.$
\end{myexample}

\begin{myexample}
Let $C=(2,0,2)$ be some possible configuration of $\Pi.$  This will bring us to the following system of linear equations using Corollary \ref{Ck-by-C0}:

\begin{tabular}{lcl}
$-x-2y+z$ & $=$ & $0$ \\
~ $x+y-z$ & $=$ & $-1$ \\
~ $x+y+z -w-2v$ & $=$ & $1$ 
\end{tabular}

This will yield solutions $(x,y,z,w,v) = (r,1,r+2,, 2(r-s +1), s),$ for some integers $r,$ and $s,$ where $r \geq s.$
 
Let $r=1, s=0,$ with $\overline{s} = (1,1,3,4,0).$  Table below show that $(2,0,2)$ cannot be reachable in this instance.
$$
\begin{tabular}{|c|c|c|c|}
\hline
$i$ \quad & \quad $\overline{s}^{(i)} - Sp^{(i)}$ \quad  & \quad $Sp^{(i)}$ \quad & $C^{(i)}$ \\
\hline
 $0$ & $(1,1,3,4,0)$ & $(1,0,1,1,0)$ & $(2,1,1)$\\
 
 $1$ & $(0,1,2,3,0)$ & $(1,0,1,0,1)$ & $(2,1,2)$ \\
 
 $2$ & $(-1,1,1,3,-1)$ &  & $
(2,1,2)$  \\
& (not a valid sum vector)& & \\
 \hline
\end{tabular}
$$

It is not hard to do the same procedure to obtain a negative results if $C^{(0)}= (0,1,1,10).$

Now, let  $r=1, s=1,$ with $\overline{s} = (1,1,3,2,1).$  From $C^{(0)}=(1,0,1,1,0),$ our table would reveal that $C=(2,0,2)$ a sum vector with negative entry, thus $(2,0,2)$ cannot be reachable.  We show below the case when $C^{(0)}= (0,1,1,1,0).$
$$
\begin{tabular}{|c|c|c|c|}
\hline
$i$ \quad & \quad $\overline{s}^{(i)} - Sp^{(i)}$ \quad  & \quad $Sp^{(i)}$ \quad & $C^{(i)}$ \\
\hline
 $0$ & $(1,1,3,2,1)$ & $(0,1,1,1,0)$ & $(2,1,1)$\\
 
 $1$ & $(1,0,2,1,1)$ & $(0,0,1,0,1)$ & $(1,1,2)$ \\
 
 $2$ & $(1,0,1,1,0)$  &   &  $(2,0,1)$ \\
& (not a valid spiking vector)& & \\
 \hline
\end{tabular}
$$
At the last row of the preceding table, the candidate spiking vector $(1,0,1,1,0)$ means $r_3$ must be used, but the configuration $(2,0,1)$ indicates that $\sigma_2$ which contains $r_3$ does not have any spikes!
Therefore, $(2,0,2)$ is not reachable in $\Pi,$ and further is not a valid configuration in $\Pi.$
\end{myexample}


\begin{mycoro}\label{reachability-cor}
Let $\Pi$ be an SN P system without delay with $n$ total rules and $m$ neurons. Let $C^{(k)},$ and $C^{(k+v)}$ be configurations reachable in $\Pi.$ Then $C^{(k+v)}$ is $v$-reachable from $C^{(k)}$ if and only if there exists a vector $\overline{s'},$ such that $$C^{(k+v)} = C^{(k)} +\overline{s'} \cdot M_{\Pi},$$ where $M_{\Pi}$ is the spiking transition  matrix of $\Pi.$
\end{mycoro}
\begin{proof}
Since $C^{(k)}$ and $C^{(k+v)}$ are $k$ and $k+v$ reachable, respectively, in $\Pi,$, Theorem \ref{reachability} implies:

\begin{tabular}{lcl}
$C^{(k+v)}$ & $=$ & $C^{(0)} +\left ( \displaystyle \sum_{i=0}^{k+v-1}Sp^{(i)} \right )\cdot M_{\Pi}.$\\
$C^{(k)}$ & $=$ & $C^{(0)} +\left ( \displaystyle \sum_{i=0}^{k-1}Sp^{(i)} \right )\cdot M_{\Pi}.$
\end{tabular}

Then $$C^{(k+v)} - C^{(k)} =\left ( \displaystyle \sum_{i=k}^{k+v-1}Sp^{(i)} \right )\cdot M_{\Pi}.$$

Let  $\overline{s'}=\displaystyle \sum_{i=k}^{k+v-1}Sp^{(i)}.$
Hence, by  Theorem \ref{reachability}, $C^{(k+v)}$ is reachable from $C^{(k)}$ if and only if $\sum_{i=k}^{k+v-1}Sp^{(i)} \in \mathbb{Z}_{k+v}^{m},$ such that $C^{(k+v)} =C^{(k)} + \left ( \displaystyle \sum_{i=k}^{k+v-1}Sp^{(i)} \right )\cdot M_{\Pi}.$
\myqed
\end{proof}


\section{SN P Systems with Delays}\label{snp+d}
In this section, let $d > 0,$ that is, we allow delays in applying rules of $\Pi$ in the neurons. All SN P systems  in this section are SN P systems with delays, unless stated otherwise.  


\begin{myexample}\label{snp+delay} {\bf A 3-Neuron SN P system from} \cite{carandang2017}\\
$\Pi = (\{a\}, \sigma_1, \sigma_2, \sigma_3, sys ),$ where $\sigma_1 = ( 1, R_1),$ with $R_1 = \{ a / a \rightarrow a \};$  $\sigma_2 = (0, R_2),$ with $R_2 = \{a / a \rightarrow a,  a^2 \rightarrow \lambda\};$ and $\sigma_3 =(1, R_3),$ with $R_3 = \{ a \rightarrow a  \colon 2\};$ $syn = \{(1,2), (2,1),  (2,3), (3,2)\}.$
\end{myexample}

This is the example used in \cite{carandang2017}  demonstrating the method of computation of SN P system with delay using matrices and matrix/vector operations.

In \cite{carandang2017}, a matrix representation was proposed for creating an implementation of SN P system with delay.   The following vectors and matrices are defined \cite{carandang2017} in addition to those used in the case where there is no delay.


\begin{mydef}\label{PM} {\bf (Production matrix)}\\ 
For an SN P system $\Pi,$ with $n$ total rules and $m$ neurons, a {\bf production matrix} of $\Pi$ is given by 
$$PM_{\Pi} = [ p_{ij}]_{n \times m},$$ where for each rule $r_i \colon E / a^{c} \rightarrow a^{p} ; d \in \sigma_j$ we have 
\[ p_{ij} =
  \begin{cases}
    p,       & \quad \text{if } r_i \in \sigma_s , (s \not = j \text{ and } (s,j) \in syn)  \text{ produced } p \text{ spikes}; \\
    0,  & \quad \text{otherwise}
  \end{cases}
\]
\end{mydef}

In the case of Example \ref{snp+delay}, we have 
$
PM_{\Pi} = \left (
 \begin{matrix} 
  0& 1 & 0 \\
  1 & 0 & 1 \\
  0 & 0 & 0 \\ 
  0 & 1 & 0
 \end{matrix}
\right )
$

\begin{mydef}\label{CM} {\bf (Consumption matrix)}\\ 
For an SN P system $\Pi,$ with $n$ total rules and $m$ neurons, a {\bf consumption matrix} of $\Pi$ is given by
$$CM_{\Pi} = [ c_{ij}]_{n \times m},$$ where for each rule $r_i \colon E / a^{c} \rightarrow a^{p} ; d \in \sigma_j$ we have 
\[ c_{ij} =
  \begin{cases}
    c,       & \quad \text{if } r_i \in \sigma_j,  \text{ consumed } c \text{ spikes}; \\
    0,  & \quad \text{otherwise}
  \end{cases}
\]
\end{mydef}

A consumption matrix for SN P system in Example \ref{snp+delay} is
$
CM_{\Pi} = \left (
 \begin{matrix} 
  1 & 0 & 0 \\
  0 & 1 & 0 \\
  0 & 2 & 0 \\ 
  0 & 0 & 1 
 \end{matrix}
\right )
$


\begin{mydef}{\bf (Delay vector)}\\
The {\bf delay vector} $D = (d_1, d_2, \ldots, d_n)$ indicates the amount of delay $d_i \geq 0$ of rule $r_i,$ for each $i= 1,2, \ldots , n.$

At some time $k,$ we denote by $DSt^{(k)} =(dst^{(k)}_1, dst^{(k)}_2, \ldots , dst^{(k)}_n)$ the $k$th {\bf delay status vector.}   $D=DSt^{(1)},$ and $DSt^{(0)}= \text{\bf 0}$ (means no delay).
\end{mydef}

The vector $DSt^{(k)}$ provides information on the delay status of each rule at time $k.$


\begin{mydef} {\bf (Status vector)}\\
The $k$th {\bf status vector} is denoted by $St^{(k)} = (st_1^{(k)}, st_2^{(k)}, \ldots , st_m^{(k)}),$ where
\[ st_{j}^{(k)} =
  \begin{cases}
    1,       & \quad \text{if } \sigma_{j} \text{ is open,} \\
    0,  & \quad \text{otherwise}
  \end{cases}
\]

When all neurons are open at time $k$, we have $St^{(k)}= \text{\bf 1}.$ This means that $st_{j}^{(k)} =1,$ for all $j=1,2, \ldots m.$    At $C^{(0)},$  $St^{(0)} = \text{\bf 1}.$ 
\end{mydef}

The status vector at time $k$ indicates which neurons are close and which are open with respect to $DSt^{(k)}.$  


\begin{mydef} {\bf (Indicator vector)}\\
The {\bf indicator vector} at time $k$ is a vector $Iv^{(k)}=( iv_1^{(k)}, iv_2^{(k)}, \ldots , iv_n^{(k)} )$ where 
\[ iv_{i}^{(k)} =
  \begin{cases}
    1,       & \quad \text{if } r_i \text{ would (be applied to) produce a spike at time } k \\
    & \quad  \text{ or } (dst^{(k)}_{i} =0.); \\
    0,  & \quad \text{otherwise}
  \end{cases}
\]
\end{mydef}

Note that the initial configuration of the SN P systems in Example \ref{snp+delay} is $C^{(0)} = (1,0,1),$  and the delay vector is $D = (0,0,0,2).$  Its $Iv^{(0)} = (1,0,0,0),$ since $\sigma_3$ has only one rule and with delay.  At the start of the computation all neurons are open, thus  $St^{(0)} = (1,1,1).$  The corresponding spiking vector is $Sp^{(0)} = (1,0,0,1).$ 

At $C^{(0)},$ all neurons are open and all possible rules can produce spikes, when chosen. If one of the chosen rules, say $r_i$ has delay $d_i$ in some neuron $\sigma_j,$ then spiking vector involving $r_i$ will have entry $sp_i^{(0)} =1,$ but $iv_i^{(0)}=0,$ since it is not allowed to produce spike because of delay $d_i$ in $r_i.$  The spike it fired will be posponed (or could perhaps be retained in $\sigma_j$) to be sent to appropriate neuron(s) until $\sigma_j$ opens again.  

When a neuron $\sigma_j$ is open any rule $r_i$ without delay chosen at time $k$ could spike.  In this case, $r_i$ could spike at time $k,$ and $iv_i^{(k)} = sp_i^{(k)} =1.$  If a rule $r_i$ in neuron $\sigma_j$ is chosen at time $k,$ with a delay $d_i$, then $r_i$ will spike at time $k,$ and $\sigma_j$ will be closed until time $k + d_i.$ Neuron $\sigma_j$ cannot receive and accept spikes until it opens again at time $k + d_i +1,$ and until then $iv_i^{(k)}=0.$  




\section{Computation of SN P system with delay via $M_{\Pi}$}\label{snp+d-compute}


\begin{mydef}\label{gainV} {\bf (Gain vector)} \\
We denote by $Gv^{(k)} = ( gv_1^{(k)}, gv_2^{(k)}, \ldots , gv_m^{(k)})$ the {\bf gain vector} of $\Pi$ at time $k,$ where $gv^{(k)}_j$ is the amount of spikes sent to $\sigma_j$ by the connected neighboring neurons at time $k,$ for $j =1,2, \ldots m.$
\end{mydef}


\begin{mylem}\label{gain}
Let $\Pi$ be an SN P systems with delay, $d >0.$ $$Gv^{(k)} = Iv^{(k)} \cdot PM.$$ 
\end{mylem}
\begin{proof}
Let $PM_{\Pi} = [ p_{ij} ]_{n \times m}$ be the production matrix of SN P system $\Pi.$ Suppose at time $k,$ $Iv^{(k)}=( iv_1^{(k)}, iv_2^{(k)}, \ldots , iv_n^{(k)} )$ is the indicator vector for $\Pi.$  Then  
$$Iv^{(k)} \cdot PM_{\Pi} =  ( x_1^{(k)}, x_2^{(k)}, \ldots , x_m^{(k)}),$$ where $x_j^{(k)} = \displaystyle \sum_{i =1}^{n} iv_i^{(k)} \cdot p_{ij},$ for $j = 1,2, \ldots , m.$  The product 
\[iv_i^{(k)} \cdot p_{ij}  =
  \begin{cases}
    p_{ij},       & \quad \text{if } iv_i^{(k)} =1; \\
    0,  & \quad \text{otherwise}
  \end{cases}
\] 
Thus, $x_j^{(k)}$ denotes the total amount of spikes produced and sent at time $k$ to $\sigma_j.$ Hence, $x_j^{(k)} = g_j^{(k)}$ in Definition \ref{gainV}.  Therefore, $$Gv^{(k)} = \left [ \displaystyle \sum_{i =1}^{n} iv_i^{(k)} \cdot p_{ij} \right ]_{1 \times m} = Iv^{(k)} \cdot PM.$$ 
\myqed
\end{proof}


\begin{mydef}\label{lossV} {\bf (Loss vector)}\\
The {\bf loss vector} at time $k$ is denoted by vector $Lv^{(k)} = (lv_1^{(k)}, lv_2^{(k)}, \ldots , lv_m^{(k)}),$ where $lv^{(k)}_i$ is the amount of spikes consumed by $\sigma_j$ at time $k,$ for each $j=1,2, \ldots , m.$
\end{mydef}


\begin{mylem}\label{loss}
Let $\Pi$ be an SN P systems with delay, $d >0.$ $$Lv^{(k)} = Sp^{(k)} \cdot CM.$$ 
\end{mylem}
\begin{proof} The proof follows the same procedure as in Lemma \ref{gain}, only we use consumption matrix $CM_{\Pi},$ instead of $PM_{\Pi}.$  Realized that the sum of products $sp_i^{(k)} \cdot c_{ij},$ represents the total spikes consumed by the system $\Pi$ at time $k.$    By Definition \ref{lossV} the vector $$\left [\displaystyle \sum_{i =1}^{n} sp_i^{(k)} \cdot c_{ij} \right ]_{1 \times j} = Lv^{(k)}.$$ 

\myqed
\end{proof}

For an SN P systems $\Pi$ with delay, $d >0,$ configurations $C^{(k+1)}$ and $C^{(k)},$  the  difference $C^{(k+1)} - C^{(k)}$ denotes the net gain (in spikes) of the neurons in $\Pi$ after transitioning from the $k$th time (step) to $(k + 1)$th time (step). Since our SN P system $\Pi$ has delay(s), not all rules will be fired, and not all neuron $\sigma_j$ will be open at some time $k.$   Therefore, we need to check on the status of neurons (open or close) at time $k,$ before finally obtaining the net gain in spikes of the system $\Pi.$   We can do this by component-wise multiplying $St^{(k)}$ to $PM_{\Pi}.$  Finally, we remove from $St^{(k)} \odot PM_{\Pi},$  vector $Sp^{(k)} \cdot CM.$  The following results are obtained.

\begin{mycoro}\label{NetGain+delay} {\bf (Net gain vector)} \cite{carandang2017}\\
Let $\Pi$ be an SN P systems with delay, $d >0.$ Then $$NG^{(k)} = St^{(k)} \odot Gv^{(k)} - Lv^{(k)},$$ 
\end{mycoro}

Since  $C^{(k+1)} - C^{(k)} = St^{(k)} \odot Gv^{(k)} - Lv^{(k)},$ we have the following Theorem:


\begin{mythm}\label{NextConfigDelay} {\bf (Next configuration with delay)} \cite{carandang2017}\\
$$C^{(k+1)} = C^{(k)} + St^{(k)} \odot (Iv^{(k)} \cdot PM) - Sp^{(k)} \cdot CM.$$
\end{mythm}


We provide here illustration of computations of SN P system with delay in Example \ref{snp+delay} from \cite{carandang2017}.  We will use the matrices $PM,$ and $CM$ of the SN P system in Example \ref{snp+delay} below.

\begin{myexample}
Let $C^{(0)}= (1,0,1)$ be the initial configuration of our SN P system with delay.  The delay vector is $D= (0,0,0,2), DSt^{(0)} =(0,0,0,0).$   Only $r_{4}$ in $\sigma_3$ has delay in the system.  At $k=0,$ we ignore delay, thus $St^{(0)}=(1,1,1).$  The indicator vector is $Iv^{(0)} = (1,0,0,0),$ while $Sp^{(0}=(1,0,0,1).$

Using matrices $PM,$ and $CM$ appropriately, we obtain the following $GV^{(0)}= Iv^{(0)} \cdot PM = (0,10),$ and $Lv^{(0)} = Sp^{(0)} \cdot CM = (1,0,1).$  Then using Corollary \ref{NetGain+delay}, we obtain $NG^{(0)} = (-1,1,-1).$  Finally, $C^{(1)}= (1,0,1) + (-1,1,-1) = (0,1,0),$ by Theorem \ref{NextConfigDelay}.

Then from $C^{(1)}=(0,1,0),$ to compute for $C^{(2)},$ we update the (rule) delay status, indicator, (neuron) status, and spiking (rule) vectors as follows;  $DSt^{(1)} =(0,0,0,2), Iv^{(1)}= (0,1,0,0), St^{(1)} = (1,1,0),$ and $Sp^{(1)}=(0,1,0,0).$  Then $Gv^{(1)} = (1,0,1), Lv^{(1)} = (0,1,0),$ and $NG^{(1)}= (1,-1,0).$ Finally, $C^{(2)} = (1,0,0).$

Then for $C^{3)},$ we have  $DSt^{(2)}= (0,0,0,1), Iv^{(2)}= (1,0,0,0), St^{(2)} = (1,1,0),$ and $Sp^{(2)}=(1,0,0,0).$   Then $C^{(3)} = (1,0,0) + (0,1,0) \odot ((0,1,0) - (1,0,0)=(0,1,0).$

To compute for $C^{(4)},$ we have $DSt^{(3)}= (0,0,0,0), Iv^{(3)}= (0,1,0,0), St^{(3)} = (1,1,1),$ and $Sp^{(3)}=(0,1,0,0).$  Performing the same computations, $C^{(4)}= (1,0,1).$

Note that $St^{(3)} = (1,1,1).$  This is because $DSt^{(3)}= (0,0,0,0),$ or all neurons become open.

$C^{(5)}= (0,1,0),$ where $DSt^{(4)}= (0,0,0,2), Iv^{(4)}= (0,1,0,0), St^{(4)} = (1,1,0),$ and $Sp^{(4)}=(0,1,0,0).$

Notice the non-zero value of the entry in the (rule) delay status vector decreases; from $DSt^{(1)}$ until $DSt^{(4)},$  spiking vectors are equal to indicator vectors.  

\end{myexample}


The next statement relates spiking transition matrix $M_{\Pi}$ of $\Pi$ with matrices $PM,$ and $CM.$ 

\begin{mycoro} \label{PM-CM-M} For any SN P system $\Pi,$
$$PM_{\Pi}  - CM_{\Pi}  = M_{\Pi},$$
where $M_{\Pi}$ is the {\bf spiking transition matrix} of $\Pi.$ 
\end{mycoro}
\begin{proof}
This follows directly from Definition \ref{PM}, Definiton \ref{CM}, and Definition \ref{matrix-SNP-d}.
\myqed
\end{proof}

Computation of SN P system in Example \ref{snp+delay} is demonstrated in \cite{carandang2017}, together with its implementation in CUDA. 

The next result suggests yet another way of computing for $C^{(k+1)}$ other than that provided in Theorem \ref{NextConfigDelay}.  This time the application of delay is done immediately unlike the standard application of it, that is at $C^{(0)}$ all neurons are open, thus $St^{(0)}= \text{\bf 1}.$  Below, we use $St^{(k+1)}, k \geq 0,$ thus we start with $St^{(1)},$ instead. 
   

\begin{mythm} \label{new}
Let $\Pi$ be an SN P systems with delay, $d >0.$ Then for $k \geq 0,$
$$C^{(k+1)} = St^{(k+1)} \odot ( C^{(k)} + Iv^{(k)} \cdot M_{\Pi}). $$
\end{mythm}
\begin{proof}   
Let $\Pi$ be an SN P systems with delay, $d >0.$   
Given $C^{(k)},$ for some $k.$   We could specify $Iv^{(k)}$ to realize which rules could be allowed to produce spike(s) at time $k.$  If we apply this vector to $PM_{\Pi},$ we will produce a gain vector that represents the amount of spikes produced by the respective rules at time $k.$  Similarly, we can apply the same $Iv^{(k)}$ to $CM_{\Pi}$ to produce the amount of spikes consumed by the set of rules that are allowed to produce spikes.  Thus, $Iv^{(k)} \cdot PM_{\Pi} - Iv^{(k)} \cdot CM_{\Pi}$ represents the spikes net gained by $\Pi$ at time $k.$  Then $C^{(k)} + Iv^{(k)} \cdot PM_{\Pi} - Iv^{(k)} \cdot CM_{\Pi}$ will be the next possible configuration after $C^{(k)}.$  Since $\Pi$ is an SN P system with delay(s), then we need to check on the status of each neuron with respect to the delay counts at time $k.$  This is to say that if a neuron is able to produce or consume spikes but it is close, then this action will be disregarded when $St^{(k)}$ is applied, otherwise the production and consumption will be considered for the next configuration.  But in this case, we right away apply the delay imposed by the selected rules at time $k.$  Then we have $C^{(k+1)} = St^{(k+1)} \odot(C^{(k)} + Iv^{(k)} \cdot PM_{\Pi} - Iv^{(k)} \cdot CM_{\Pi}).$  Corollary \ref{PM-CM-M}, we have 
$$C^{(k+1)} = St^{(k+1)} \odot ( C^{(k)} + Iv^{(k)} \cdot M_{\Pi}). $$
\myqed
\end{proof}


We provide below illustrations of the use of Theorem \ref{new}.
\begin{myexample}
Using the SN Psystem with delay from Example \ref{snp+delay}.  By Corollary \ref{PM-CM-M}, we have $M_{\Pi}= \left (
 \begin{matrix} 
  -1 & 1 & 0 \\
  1 & -1 & 1 \\
  0 & -2 & 0 \\ 
  0 & 1 & -1 
 \end{matrix}
\right ).
$  By Theorem  \ref{new},  $C^{(1)}= St^{(1)} \odot (C^{(0} + Iv^{(0} \cdot M_{\Pi}).$   Then $C^{(1)} = (1,1,0) \odot ( (1,0,1) + (1,0,0,0)  \cdot  \left (
 \begin{matrix} 
  -1 & 1 & 0 \\
  1 & -1 & 1 \\
  0 & -2 & 0 \\ 
  0 & 1 & -1 
 \end{matrix}
\right ) )
 = (0,1,0)$
 \end{myexample}

\begin{myexample}
Let $C^{(2)}= (1,0,0)$ configuration of  $\Pi$ of Example \ref{new}. Then

$C^{(3)} = (1,1,1) \odot ( (1,0,0) + (1,0,0,0)  \cdot  \left (
 \begin{matrix} 
  -1 & 1 & 0 \\
  1 & -1 & 1 \\
  0 & -2 & 0 \\ 
  0 & 1 & -1 
 \end{matrix}
\right ) )
 = (0,1,0)$
\end{myexample}

Values of $C^{(1)},$ and $C^{(3)}$  are equal via Theorem \ref{new} and Theorem \ref{NextConfigDelay}. 

 
\begin{mycoro}
Let $\Pi$ be an SN P systems such that $St^{(k)}= \text{\bf 1},$ for $k \geq 1.$  Then
$$
C^{(k+1)} =  C^{(k)} + Sp^{(k)} \cdot M_{\Pi}. $$
\end{mycoro}
\begin{proof} If $St^{k)} = \text{\bf 1},$ means $Sp^{(k)} = Iv^{(k)},$ for all $k.$  Then $St^{(k+1)} \odot ( C^{(k)} + Iv^{(k)} \cdot M_{\Pi}) =  C^{(k)} + Iv^{(k)} \cdot M_{\Pi}.$   
\myqed
\end{proof}




To obtain result similar to Theorem \ref{reachability} we first note the following:

Suppose we have $St^{(k)} = (st_1^{(k)}, st_2^{(k)}, \ldots , st_m^{(k)}),$ for some $k.$  This vector indicates the status of each neuron at time $k.$  We define the following {\bf rule status vector} $RSt^{(k)}=( rst_1^{(k)}, rst_2^{(k)}, \ldots , rst_n^{(k)} ),$ such that 
\[ rst_{i}^{(k)} =
  \begin{cases}
    1,       & \quad \text{if } r_i  \in \sigma_j \text{ and } \sigma_j \text{ is open at time } k; \\
    0,  & \quad \text{otherwise}
  \end{cases}
\]

Suppose we have the following vectors from $\Pi;$ $Iv^{(k)}=( iv_1^{(k)}, iv_2^{(k)}, \ldots , iv_n^{(k)} ),$ $St^{(k)} = (st_1^{(k)}, st_2^{(k)}, \ldots , st_m^{(k)}),$ and $RSt^{(k)}=( rst_1^{(k)}, rst_2^{(k)}, \ldots , rst_n^{(k)} )$ and marix $M_{\Pi}.$  Then the following are easily verifiable:
$$St^{(k)} \odot (Iv^{(k)} \cdot M_{\Pi}) =  (RSt^{(k)} \odot Iv^{(k)}) \cdot M_{\Pi}.$$


\begin{myexample}  We use SN P system $\Pi$ of Example \ref{snp+delay}.
If $St^{(1)}= (1,1,0),$ then $RSt^{(1)}= (1,1,1,0),$ and $Iv^{(1)}=(1,0,0,0).$  Given $M_{\Pi}= \left (
 \begin{matrix} 
  -1 & 1 & 0 \\
  1 & -1 & 1 \\
  0 & -2 & 0 \\ 
  0 & 1 & -1 
 \end{matrix}
\right ),
$  we have  $(1,1,0) \odot ((1,0,0,0) \cdot 
M_{\Pi} )
  =  ((1,1,1,0) \odot (1,0,0,0)) \cdot   M_{\Pi} )
 = (-1,1,0)$.
\end{myexample}
 
 This equality could lead us to the following results.

\begin{mythm}\label{reachability-delay}
Let $\Pi$ be an SN P systems with delay, $d >0.$ Then for $k \geq 0,$
$$C^{(k+1)} = \left  (\bigodot_{i=1}^{k+1} St^{(i)} \right ) \odot C^{(0)} + \left (\displaystyle\sum_{j=0}^{k} \bigodot_{i=j+2}^{k+1} St^{(i)} \odot RSt^{(j+1)} \odot Iv^{(j)} \right ) \cdot M_{\Pi}, $$
where 
$C^{(0)}$ is the initial configuration of $\Pi.$
\end{mythm}
\begin{proof}
The proof follows the same procedure done for Corollary \ref{Ck-by-C0}, which may not be hard but more laborious.
\myqed
\end{proof}

Theorem \ref{reachability-delay}  implies reachability of a configuration $C$ in SN P system $\Pi$ with delay.


\section{Final Remarks}\label{final}
We revisited matrix representation of SN P systems without delay and with delay as presented in \cite{zeng-etal2010}, and \cite{carandang2017},  respectively.  We introduced a square matrix called $\text{\bf Struc-}M_{\Pi}$ of SN P system $\Pi$ without delay.  We observed further, that if the $\text{\bf row-rank}(\text{\bf Struc-}M_{\Pi}) < m,$ where $m$ is the number of neurons, a directed cycle would exists in $\Pi.$    These observations have some possible implications on the periodicity investigation initiated in \cite{Ibo2011}.  The matrix $M_{\Pi}$ of an arbitrary SN P system $\Pi$ is general not a square matrix.  

Most of the effort of doing matrix representation of SN P system and other variants of P systems was motivated by how it could be implemented in a computer.  It would perhaps be a nice task to see how and what would be the implications of some properties of matrices to SN P systems. Could it be possible to classify  SN P system using properties of matrices, example invertibility or non-sigularity, among others? 

We define reachability of configuration $C$ in SN P systems with and without delay.  In the case of SN P system without delay, reachability of $C$ depends on the existence of the solutions to the system of linear equations $C - C^{(0)} = \overline{s} \cdot M_{\Pi}$ in $\mathbb{Z}_{k}^{m},$ which is a sum of valid spiking vectors of the system $\Pi.$  This could be also the case, if we want to know if $C^{(k+v)}$ is reachable from $C^{(k)},$  for some integers $k$ and $v.$  

Results in Section \ref{snp-d-reachability} have already been observed in \cite{Ibo2011} only on a a very specific SN P system where there is only one rule per neuron in order to force $M_{\Pi}$ to be a square matrix.  Corollary \ref{reachability-cor} is a generalization of Lemma $2$ in \cite{Ibo2011}.    Eventually, it is an on-going task investigating periodicity in SN P systems $\Pi$ where $M_{\Pi}$ is not necessarily a square matrix.  

Theorem \ref{reachability-delay} gives a version of reachability of configurations of  SN P system with delay.   

It is a goal to further study these matrix representations not only of SN P systems but also other variants of P systems. Already in the 2011,  \cite{juayong2011} suggested matrix representation for ECPe systems.  

Theorem \ref{new} provides another way of computing next configuration of SN P system with delay using only  $M_{\Pi}.$  In the point of view of implementation, these results provide a more efficient use of computer space/memory.  Note that \cite{carandang2017} used two $n \times m$ matrices  ($PM_{\Pi},$ and $CM_{\Pi}$) and four vectors namely, $Sp^{(k)},$ and $Iv^{(k)}$ (which are all $1 \times n$ vectors) and $St^{(k)},$ and $C^{(k)}$ (which are $1 \times m$ vectors).  Theorem \ref{new} uses only $M_{\Pi}$ (an $n \times m$ matrix) and vectors $Iv^{(k)}$ (a $1 \times n$ vector), $St^{(k)},$ and $C^{(k)}$ (which are $1 \times m$ vectors). 

Finally, to demonstrate practically the claimed space and time efficiency of Theorem \ref{new}, we suggest to construct a software implementation using this result. 

\pagebreak

\subsection*{Acknowledgement}
The author would like to thank the support from DOST-ERDT research grants; Semirara Mining Corp. Professorial Chair for Computer Science of College of Engineering, UPDiliman; RLC grant from UPD-OVCRD.


\end{document}